\documentclass[twoside,11pt]{article}
\usepackage{pgm}

\usepackage{hyperref,color,soul,booktabs}
\setulcolor{blue}
\newcommand{\BibTeX}{\textsc{B\kern-0.1emi\kern-0.017emb}\kern-0.15em\TeX}
\usepackage{graphicx}
\usepackage{floatpag}
\usepackage{siunitx}
\usepackage{amsmath,amssymb,amsfonts}
\usepackage{thmtools}
\usepackage{thm-restate}
\usepackage{cleveref}
\usepackage{algorithmic}
\usepackage{wrapfig}
\usepackage{tikz-cd}
\usetikzlibrary{shapes}
\usetikzlibrary{arrows}
\usepackage[linesnumbered,algoruled,boxed,lined,ruled]{algorithm2e}
\usetikzlibrary{fit,calc}
\newcommand*{\tikzmk}[1]{\tikz[remember picture,overlay,] \node (#1) {};\ignorespaces}
\newcommand{\boxit}[1]{\tikz[remember picture,overlay]{\node[yshift=3pt,fill=#1,opacity=.25,fit={(A)($(B)+(.3\linewidth,.8\baselineskip)$)}] {};}\ignorespaces}
\usepackage{textcomp}
\usepackage{xcolor}
\makeatletter
\newcommand\@erelb@r[1]{%
  \mathrel{\tikz[baseline=-.5ex]\draw[#1] (0,0)--(.5,0);}
}
\newcommand{\erelbar}[1]{\@erelbar#1}
\def\@erelbar#1#2{%
  \ifcase\numexpr#1*4+#2\relax
    \@erelb@r{-}\or     
    \@erelb@r{->}\or    
    \@erelb@r{-|}\or    
    \@erelb@r{-o}\or   
    \@erelb@r{<-}\or    
    \@erelb@r{<->}\or   
    \@erelb@r{<-|}\or   
    \@erelb@r{<-o}\or   
    \@erelb@r{|-}\or    
    \@erelb@r{|->}\or   
    \@erelb@r{|-|}\or   
    \@erelb@r{|-o}\or 
    \@erelb@r{o-}\or   
    \@erelb@r{o->}\or  
    \@erelb@r{o-|}\or  
    \@erelb@r{o-o}    
  \else
    \@wrong
  \fi
}

\ShortHeadings{Learning LWF CGs: an Order Independent Algorithm}{Javidian et al.}

\begin{document}
\title{Learning LWF Chain Graphs: an Order Independent Algorithm}
%
%
%
\author{\Name{Mohammad Ali Javidian} \Email{javidian@email.sc.edu}\and
   \Name{Marco Valtorta} \Email{mgv@cse.sc.edu}\and
   \Name{Pooyan Jamshidi} \Email{pjamshid@cse.sc.edu}\\
   \addr University of South Carolina, Columbia, SC 29208, USA }
   
\maketitle              
\begin{abstract}LWF chain graphs combine directed acyclic graphs and undirected graphs. 
  We present a PC-like algorithm that finds the structure of chain graphs under the faithfulness assumption to resolve the problem of scalability of the proposed algorithm by Studen{\'y} (1997).  
  We prove that our PC-like algorithm is order dependent, in the sense that the output can depend on the order in which the variables are given. This order dependence 
  can be very pronounced in high-dimensional settings. We propose two modifications of the PC-like algorithm that remove part or all of this order dependence. Simulation results under a variety of settings demonstrate the competitive performance of the PC-like algorithms in comparison with the decomposition-based method, called  LCD algorithm, proposed by Ma et al. (2008) in low-dimensional settings and improved performance in high-dimensional settings.
\end{abstract}
\section{Introduction}
Probabilistic graphical models (PGMs) are now widely accepted as a powerful and
mature tool for reasoning and decision making under uncertainty.
A PGM is a compact representation of a joint probability
distribution, from which we can obtain marginal and conditional probabilities \citep{sucar15}. In fact, any PGM consists of two main components: (1) a graph that defines the structure of that model; and (2) a joint distribution over random variables of the model.
Two types of graphical representations of distributions are commonly used, namely, Bayesian networks and Markov networks. Both families encompass the properties of factorization and independence, but they differ in the set of independencies they can encode and the factorization of the distribution that they induce.
  
Currently systems containing both causal and non-causal relationships are mostly modeled with \textit{directed acyclic graphs} (DAGs). Chain graphs (CGs) are a type of mixed graphs, admitting both directed and undirected edges, which contain no partially directed cycles.
\begin{figure}[ht]
    \centering
    \fbox{
    \includegraphics[scale=.14]{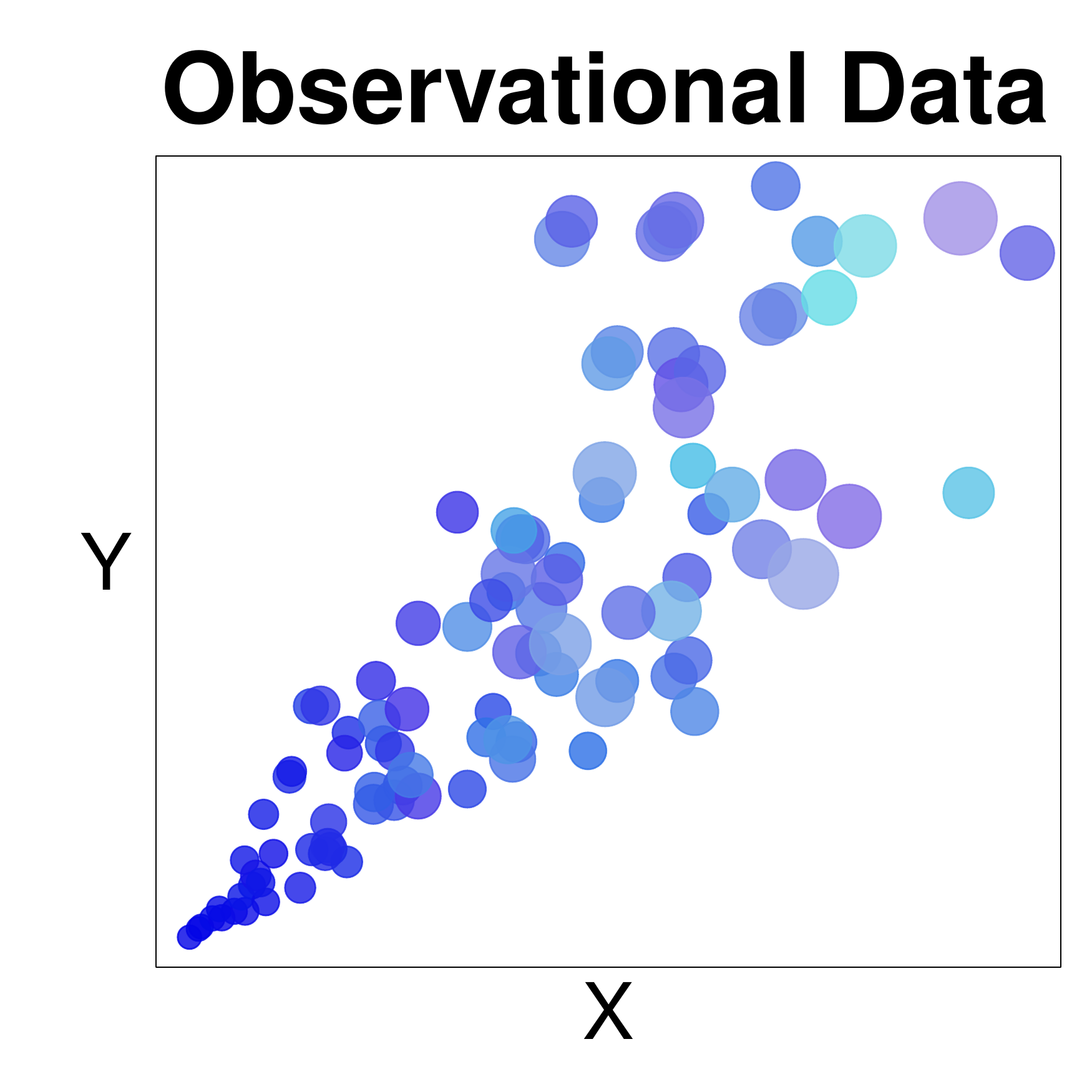}
    \put(0,30){\hbox{\includegraphics[scale=.02]{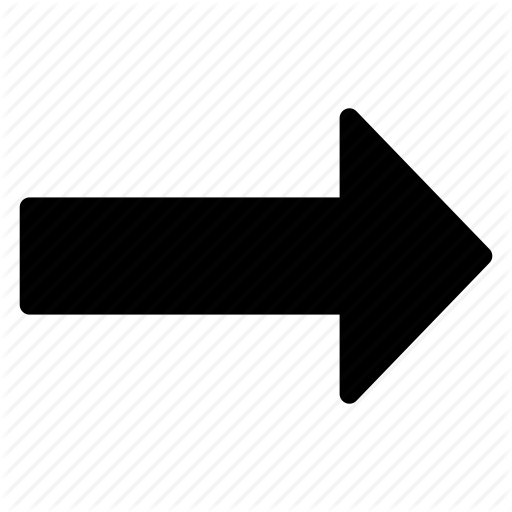}}}
 \includegraphics[scale=.13]{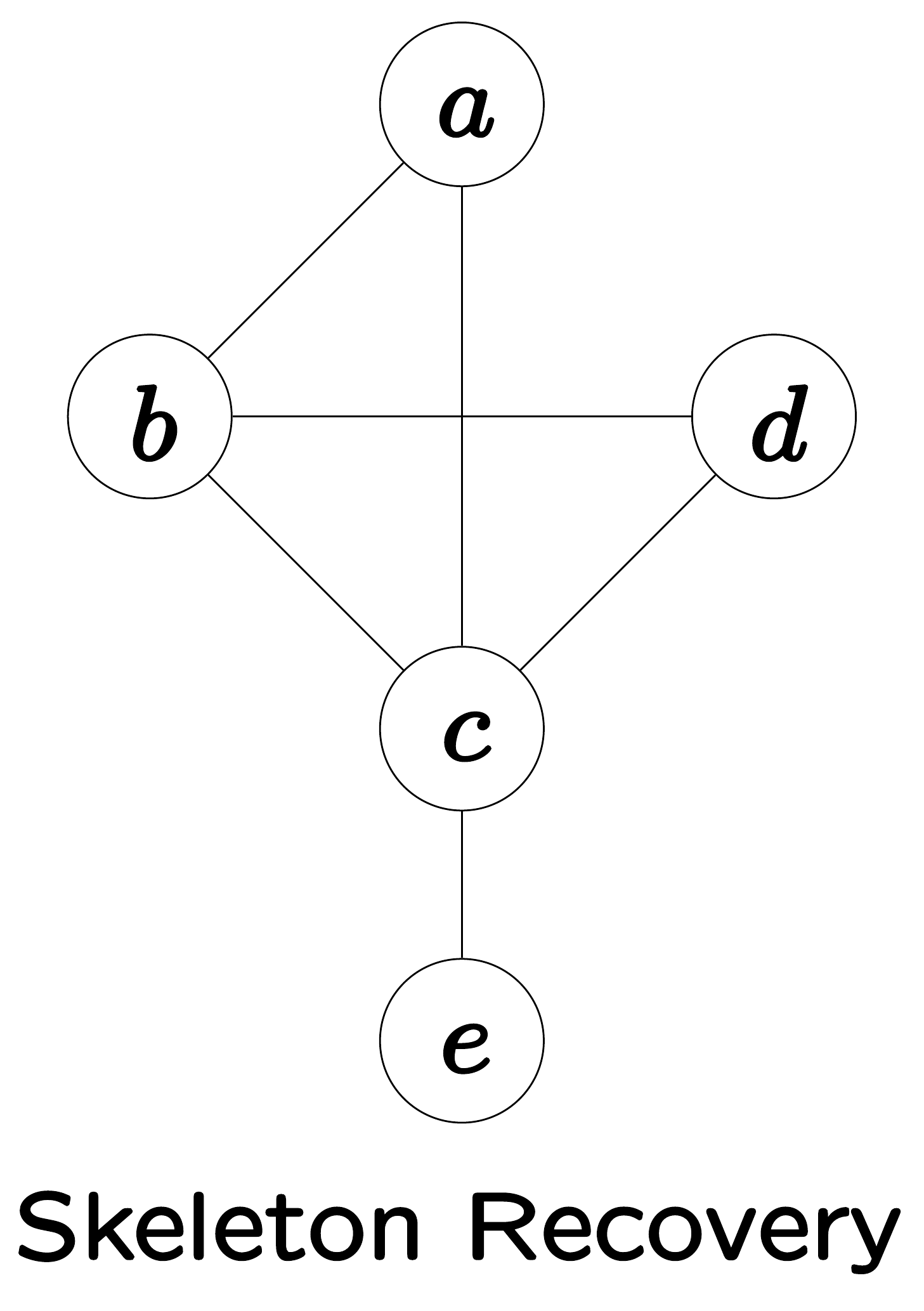} 
 \put(0,30){\hbox{\includegraphics[scale=.02]{images/rarrow.png}}}
 \includegraphics[scale=.13]{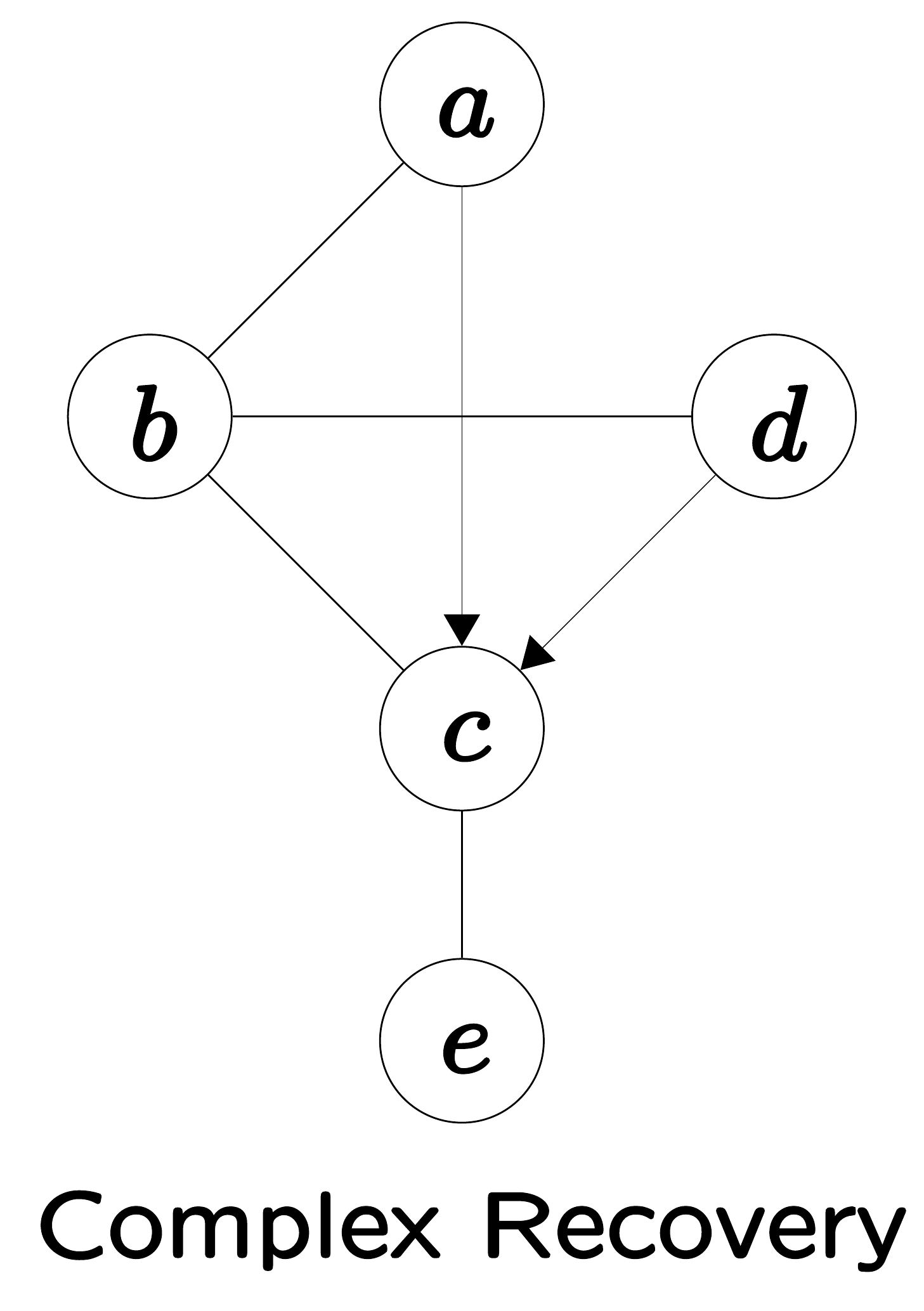}
 }
    \caption{\footnotesize{Learning LWF CGs with PC-like algorithm.}}
    \label{fig:outline}
\end{figure}
So, CGs may contain two types of edges,
the directed type that corresponds to the causal relationship in DAGs and a
second type of edge representing a symmetric relationship \citep{s2}. \textit{LWF Chain graphs} were introduced by \textbf{L}auritzen, \textbf{W}ermuth and \textbf{F}rydenberg \citep{f}, \citep{lw} as a generalization of graphical models based on undirected graphs and DAGs and widely studied e.g., in~\citep{l,d,mxg,psn,studeny97,volf,studeny09,roverato05,roverato06,jv-pgm18}, , among others.. 
From the \textit{causality} point of view, in an LWF CG directed edges represent \textit{direct causal effect}s, and undirected edges represent causal effects due to \textit{interference} \citep{Shpitser17}, \citep{Ogburn18}, and \citep{Bhattacharya19}.

One important aspect of PGMs is the possibility of learning the structure of models directly
from sampled data. Six \textit{constraint-based} learning algorithms, that use a statistical analysis to test the presence of a
conditional independency, exist for learning LWF CGs: (1) the inductive causation like (IC-like) algorithm \citep{studeny97}, (2) the decomposition-based algorithm called LCD  (\textbf{L}earn \textbf{C}hain graphs via \textbf{D}ecomposition) \citep{mxg}, (3) the answer set programming (ASP) algorithm \citep{sjph}, (4) the inclusion optimal (CKES) algorithm \citep{psn}, and (5) the local structure learning of chain graphs with the false discovery rate control \citep{Wang2019}, (5) the local structure learning of chain graphs with the false discovery rate control \citep{Wang2019}, and (6) the Markov blanket discovery (MbLWF) algorithm \citep{jvj-UAI2020}.

Similar to the \textit{inductive causation (IC)} algorithm \citep{verma90}, the IC-like algorithm \citep{studeny97} cannot be applied to large numbers of
variables because for testing whether there is a set separating $X$ and $Y$ in the skeleton recovery, the IC-like algorithm
might search all $2^{n-2}$ subsets of all $n$ random variables not including $X$ and $Y$. In order to overcome the scalability of the IC-like algorithm, we propose a constraint-based method for learning the structural of chain graphs based on the idea of the \textit{PC algorithm} proposed by \textbf{P}eter Spirtes and \textbf{C}lark Glymour \citep{sgs}, which is used for learning the structure of Bayesian networks (BNs). Our method modifies the IC-like algorithm to make it computationally feasible in the phase of skeleton recovery and to avoid the time consuming procedure of complex recovery. 
We prove that the proposed PC-like algorithm in this paper is order dependent,
in the sense that the output can depend on the order in which the variables are given. We propose several modifications of the PC-like algorithm that remove part or all of this order dependence, but do not change the result when perfect conditional independence information is used. When
applied to data, the modified algorithms are partly or fully order independent. 

Our proposed  algorithm, called the \textit{SPC4LWF} (Stable PC-like for LWF CGs), similarly to the LCD algorithm, is able to exploit parallel computations for scaling up the task of learning LWF CGs. This will enable effective LWF chain graph discovery on large/high-dimensional datasets. In fact, lower complexity, higher power of computational independence test, better learned structure quality, along with the ability of exploiting parallel computing make our proposed algorithm in this paper more desirable and suitable for big data analysis when LWF chain graphs are being used.
Our main contributions are the following:

\noindent \textbf{(1)} We propose a PC-like algorithm for learning the structure of LWF CGs under the faithfulness assumption that includes two main procedures: (i) a feasible method for learning CG skeletons following the same idea of the PC algorithm, (ii) a polynomial time procedure for complex recovery similar to the proposed approach in \citep{mxg}. The whole procedure is shown in Figure \ref{fig:outline}. The algorithm and its soundness are discussed in section \ref{sec:opclikelwf}. 

\noindent \textbf{(2)} In section \ref{sec:SPSLWF}, we show that our proposed PC-like algorithm in section \ref{sec:opclikelwf} is order dependent. Then, we propose two modifications of this algorithm that remove part or all of this order dependence. The soundness of modified algorithms are discussed in section \ref{sec:SPSLWF}.

\noindent \textbf{(3)} We experimentally compare the performance of our proposed PC-like algorithms with the LCD algorithm in section \ref{evaluation}, and we show that the PC-like algorithms are comparable to the LCD algorithm in low-dimensional settings and superior in high-dimensional settings in terms of error measures and runtime. 

\noindent \textbf{(4)} We release supplementary material (\url{https://github.com/majavid/PC4LWF2020}) including data and an R package that implements the proposed algorithms.

\section{Definitions and Concepts}
Below, we briefly list some of the central concepts used in this paper
(see \citep{l} for more details).
In this paper, we consider graphs containing both directed ($\to$) and undirected ($-$) edges
and largely use the terminology of \citep{l}, where the reader can also find further
details. Below we briefly list some of the central concepts used in this paper.

A \textit{path} in $G$ is a sequence of its distinct nodes $v_1,v_2,\dots,v_k, k\ge 1,$ such that $\{v_i,v_{i+1}\}$ is an edge in $G$ for every $i = 1,\dots, k - 1$. It is called a \textit{cycle} if $v_{k+1}\equiv v_1$, and $k\ge 3$. 
A \textit{chord} of a cycle $C$ is an edge not in $C$ whose endpoints 
lie in $C$. A \textit{chordless cycle} in $G$ is a cycle of length at least 4 in $G$ that has
no chord (that is, the cycle is an induced subgraph). A cycle of length 3 is both chordal and chordless.
A \textit{partially directed cycle} (or semi-directed cycle) in a graph $G$ is a sequence of $n$ distinct vertices $v_1,v_2,\dots,v_n (n\ge 3)$, and $v_{n+1}\equiv v_1$, such that
(a) $\forall i (1\le i\le n)$ either $v_i-v_{i+1}$ or $v_i\to v_{i+1}$, and
(b) $\exists j (1\le j\le n)$ such that $v_j\to v_{j+1}$.

If there is a path from  $a$ to $b$ we say that $a$ leads to $b$ and write $a\mapsto b$. The vertices $a$ such that $a\mapsto b$ and $b\not\mapsto a$ are the \textit{ancestors} $an(b)$ of $b$, and the descendants $de(a)$ of $a$ are the vertices $b$ such that $a\mapsto b$ and $b\not\mapsto a$. The non-descendants are $nd(a) = V\setminus (de(a) \cup \{a\})$.  If there is an arrow from $a$ pointing towards $b$, $a$ is said to be a parent 
of $b$. If there is an undirected edge between $a$ and $b$, $a$ and $b$ are said to be adjacent or neighbors. The boundary $bd(A)$ of a subset $A$ of vertices is the set of vertices in $V\setminus A$ that are parents or neighbors to vertices in $A$. The closure of $A$ is $cl(A)=bd(A)\cup A$. If $bd(a)\subseteq A$, for all $a\in A$ we say that $A$ is an \textit{ancestral set}. The smallest ancestral set containing $A$ is denoted by $An(A)$.

An \textit{LWF chain graph} is a graph in which there are no partially directed cycles. The chain components $\mathcal{T}$ of a chain graph are the connected components of the undirected
graph obtained by removing all directed edges from the chain graph. A \textit{minimal complex} (or simply a complex or a \textit{U}-structure) in a chain graph is an induced subgraph of the form $a\to v_1-\cdots \cdots-v_r\gets b$. The \textit{skeleton} (underlying graph) of an LWF CG $G$ is obtained from $G$ by changing all directed edges of $G$ into undirected edges. For a chain graph $G$ we define its \textit{moral graph} $G^m$ as the undirected  graph with the same vertex set but with $\alpha$ and $\beta$ adjacent in $G^m$ if and 
only if either $\alpha \to \beta$, or $\alpha - \beta$, or $\beta\to \alpha$ or if there are $\gamma_1,\gamma_2$ in the same chain 
component such that $\alpha\to \gamma_1$ and $\beta\to \gamma_2$. 

\emph{Global Markov property for LWF chain graphs:} 
	For any triple
	$(A, B,S)$ of disjoint subsets of $V$ such that $S$ separates $A$ from $B$
	in $(G_{An(A\cup B\cup S)})^m$, in the moral graph of the smallest ancestral set containing $A\cup B\cup S$, we have $A \!\perp\!\!\!\perp B | S$ i.e., $A$ is independent of $B$ given $S$. We say $S$ $c$-separates $A$ from $B$ in the chain graph $G$. 
We say that two LWF CGs $G$ and $H$ are \textit{Markov equivalent}
or that they are in the same \textit{Markov equivalence class} if they induce the same  
conditional independence restrictions.  Two CGs $G$ and $H$ are Markov equivalent if and only if they have the same skeletons and the same minimal complexes \citep{f}. Every class of Markov equivalent CGs has a unique CG with the
greatest number of undirected edges. This
graph is called the \textit{largest CG} (LCG) of the corresponding class of Markov equivalent CGs \citep{f}. 

\section{PC4LWF: a PC-Like Algorithm for Learning LWF CGs}\label{sec:opclikelwf}
In this section, we discuss how the IC-like algorithm \citep{studeny97} can be modified to obtain a computationally feasible algorithm for LWF CGs recovery. A brief review of the IC-like algorithm is presented first, then we present a PC-like algorithm, called \textit{PC4LWF}, which is a constraint-based algorithm that learns a CG from a probability distribution faithful to some CG.

The IC-like algorithm \citep{studeny97} is a constraint-based algorithm proposed for LWF CGs and is based on three sequential phases. The first phase finds the
adjacencies (skeleton recovery), the second phase orients the edges that must be oriented the same in
every CG in the Markov equivalence class (complex recovery), and the third phase transforms
this graph into the largest CG (LCG recovery).

The skeleton recovery of the IC-like algorithm works as follows: construct an undirected graph $H$ such that vertices
$u$ and $v$ are connected with an undirected edge if and only if no set $S_{uv}$ can be found such that $u\!\perp\!\!\!\perp v|S_{uv}$.
This procedure is very inefficient because this requires a number of independence tests that increases exponentially with the number of vertices.
In other words, to determine whether  there is a set separating $u$ and $v$, we might search all $2^{n-2}$ subsets of all $n$ random variables excluding $u$ and $v$. So, the complexity for investigating each possible edge in the skeleton is $O(2^n)$
and hence the complexity for constructing the skeleton is $O(n^22^n)$, where $n$ is the number of vertices in the LWF CG.
Since it is enough to find one $S$ making $u$ and $v$ independent to remove the undirected edge $u\erelbar{00}v$, one obvious short-cut is to do the tests in some order, and skip unnecessary tests. In the PC algorithm for BNs the revised edge removal step is done as shown in Algorithm \ref{alg:edgremoval}.
\begin{algorithm}[t]
\caption{Edge-removal step of the PC algorithm for BNs}\label{alg:edgremoval}
    \footnotesize\For{$i\gets 0$ \KwTo $|V_H|-2$}{
        \While{possible}{
            Select any ordered pair of nodes $u$ and $v$ in $H$ such that $u\in ad_H(v)$, $|ad_H(u)\setminus v|\ge i$ ($ad_H(x):=\{y\in V| x\to y, y\to x, \textrm{ or }x-y\}$)\;
            \If{\textrm{there exists $S\subseteq (ad_H(u)\setminus v)$ s.t. $|S|=i$ and $u\perp\!\!\!\perp_p v|S$ (i.e., $u$ is independent of $v$ given $S$ in the probability distribution $p$)}}{
                Set $S_{uv} = S_{vu} = S$\;
                Remove the edge $u - v$ from $H$\;
            }
        }
    }
\end{algorithm}

Since the PC algorithm only looks at adjacencies of $u$ and $v$ in the current stage of the algorithm, rather than all possible subsets, the PC algorithm performs fewer independence tests compared to the IC algorithm.
The complexity
of the PC algorithm for DAGs is difficult to evaluate exactly, but with the \textit{sparseness assumption} the worst case is with high probability bounded by $O(n^q)$, where $n$ is the number of vertices and $q$ is the maximum number of the adjacent vertices of the true underlying DAG \citep{Kalisch07}. Our main intuition is that replacing the skeleton recovery phase in the IC-like algorithm with a PC-like approach will speed up this phase and make it computationally scalable when the true underlying LWF CG is sparse (see the skeleton recovery phase of Algorithm \ref{alg:lwfopc}).

\begin{algorithm}[t]
\caption{PC-like algorithm for LWF CGs}\label{alg:lwfopc}
	\SetAlgoLined
	\small\KwIn{a set $V$ of nodes and a probability distribution $p$ faithful to an unknown LWF CG $G$.}
	\KwOut{The pattern of $G$.}
    Let $H$ denote the complete undirected graph over $V$\;
    \small\tcc{Skeleton Recovery}
\For{$i\gets 0$ \KwTo $|V_H|-2$}{
        \While{possible}{
            Select any ordered pair of nodes $u$ and $v$ in $H$ such that $u\in ad_H(v)$ and $|ad_H(u)\setminus v|\ge i$\;
            \If{\textrm{there exists $S\subseteq (ad_H(u)\setminus v)$ s.t. $|S|=i$ and $u\perp\!\!\!\perp_p v|S$ (i.e., $u$ is independent of $v$ given $S$ in the probability distribution $p$)}}{
                Set $S_{uv} = S_{vu} = S$\;
                Remove the edge $u - v$ from $H$\;
            }
        }
    }
    \small\tcc{Complex Recovery from \citep{mxg}}
    Initialize $H^*=H$\;
    \For{\textrm{each vertex pair $\{u,v\}$ s.t. $u$ and $v$ are not adjacent in $H$}}{
        \For{\textrm{each $u-w$ in $H^*$}}{
            \If{$u\not\perp\!\!\!\perp_p v|(S_{uv}\cup \{w\})$}{
                Orient $u - w$ as $u\to w$ in $H^*$\;
            }
        }
    }
    Take the pattern of $H^*$\;
    \small\tcc{\parbox[t]{.97\linewidth}{To get the pattern of $H^*$ in line 19, at each step, we consider a pair of candidate complex arrows $u_1 \to w_1$ and $u_2\to w_2$ with $u_1 \ne u_2$, then we check whether there is an undirected path from $w_1$ to $w_2$ such that none of its intermediate vertices is adjacent to either $u_1$ or $u_2$. If there exists such a path, then $u_1 \to w_1$ and $u_2\to w_2$ are labeled (as complex arrows). We repeat this procedure until all possible candidate pairs are examined. The pattern is then obtained by removing directions of all unlabeled as complex arrows in $H^*$ \citep{mxg}.}}
\end{algorithm}

The looping procedure of the IC-like algorithm for complex recovery is computationally expensive. We use a polynomial time approach similar to the proposed algorithm by \citep{mxg} to reduce the computational cost of the complex recovery (see the complex recovery phase of Algorithm \ref{alg:lwfopc}). 
Finally, the IC-like algorithm uses three basic
rules, namely the \textit{transitivity rule}, the \textit{necessity rule}, and the \textit{double-cycle rule}, for changing the obtained pattern in the previous phase into the corresponding largest CG (see \citep{studeny97} for details). 
When we have perfect conditional independence, both IC-like and LCD algorithms recover the structure of the model correctly if the probability distribution of the data is \textit{faithful} to some LWF CGs i.e., all conditional independencies among variables can be represented by an LWF CG.
The entire process is formally described in Algorithm \ref{alg:lwfopc}. The correctness of Algorithm \ref{alg:lwfopc}  is proved in  Appendix \ref{appendixA}. 

\textbf{Computational Complexity Analysis of Algorithm \ref{alg:lwfopc}.} The complexity of the algorithm for a graph $G$ is bounded by the largest degree in $G$. Let
$k$ be the maximal degree of any vertex and let $n$ be the number of vertices. Then in the
worst case the number of conditional independence tests required by the algorithm is
bounded by
$$2\binom{n}{2}\sum_{i=0}^k\binom{n-2}{i}\le \frac{n^2(n-2)^k}{(k-1)!}$$ To derive the inequality, use induction
on $k$ \cite[p. 552]{Neapolitan}. So, Algorithm \ref{alg:lwfopc} has a worst-case running time of $O(n^{k+2})$.
This is a loose upper bound even in the worst case; it assumes that in the worst case for $n$
and $k$, no two variables are \textit{c}-separated by a set of less than cardinality $k$, and for many values of $n$ and $k$ we have been unable to find graphs with that property. The worse case is rare, and the average number of conditional independence tests required for graphs of maximal
degree $k$ is much smaller. In practice it is possible to recover sparse graphs with as many
as a hundred variables as shown in section \ref{evaluation}.

\section{STABLE PC-LIKE ALGORITHM}\label{sec:SPSLWF}
In this section, we show that the PC-like algorithm proposed in the previous section is order dependent,
in the sense that the output can depend on the order in which the variables are given. Proof of theorems in this section can be found in Appendix \ref{appendixB}.

In applications, we do not have perfect conditional independence information.
Instead, we assume that we have an i.i.d. sample of size $n$ of variables $V = (X_1,\dots,X_p)$. In the PC-like algorithm all conditional independence queries are estimated by statistical conditional independence tests at some pre-specified significance level (p value) $\alpha$. For example, if the distribution of $V$ is multivariate Gaussian, one can test for zero partial correlation, see, e.g., \citep{Kalisch07}. Hence, we use the $\mathsf{gaussCItest()}$ function from the R package \href{https://cran.r-project.org/web/packages/pcalg}{$\mathsf{pcalg}$} throughout this paper. Let order($V$) denote an ordering on the variables in $V$. We now consider the role of
order($V$) in every step of the Algorithm \ref{alg:lwfopc}.

In the skeleton recovery phase of the PC-like algorithm, the order of variables affects the estimation of the skeleton and the separating sets. In particular, as noted for the special case of BNs in~\citep{Colombo2014}, for each level of $i$, the order of variables determines the order in which pairs of adjacent
vertices and subsets $S$ of their adjacency sets are considered (see lines 4 and 5 in Algorithm \ref{alg:lwfopc}). The skeleton $H$ is updated after each edge removal. Hence, the adjacency sets typically change within one level of $i$, and this affects which other conditional independencies are
checked, since the algorithm only conditions on subsets of the adjacency sets. When we have perfect conditional independence information,  all orderings on the variables lead to the same output. In the sample version, however, we typically make
mistakes in keeping or removing edges, because conditional independence relationships have to be estimated from data. In such cases, the resulting changes in the adjacency
sets can lead to different skeletons, as illustrated in Example \ref{ex1OrderDepLWF}.

Moreover, different variable orderings can lead to different separating sets in the skeleton recovery phase.
When we have perfect conditional independence information, this is not important, because any valid separating set leads to the
correct \textit{U}-structure decision in the complex recovery phase. In the sample version, however, different separating
sets in the skeleton recovery phase may yield different decisions about \textit{U}-structures in the complex recovery phase.
This is illustrated in Example \ref{ex2OrderDepLWF}.

\begin{example}\textbf{(Order dependent skeleton of the  PC4LWF algorithm.)}\label{ex1OrderDepLWF}
Suppose that the distribution of $V = \{a,b,c,d,e\}$ is faithful to the DAG in Figure
\ref{fig:OrderDepex1LWF}(a). This DAG encodes the following conditional independencies with minimal separating sets: $a\perp\!\!\!\perp d|\{b,c\}$ and $a\perp\!\!\!\perp e|\{b,c\}$.
Suppose that we have an i.i.d. sample of $(a,b,c,d,e)$, and that the following
conditional independencies with minimal separating sets are judged to hold at some significance level $\alpha$: $a\perp\!\!\!\perp d|\{b,c\}$, $a\perp\!\!\!\perp e|\{b,c,d\}$, and $c\perp\!\!\!\perp e|\{a,b,d\}$. Thus, the first two are correct, while the third is false.

We now apply the skeleton recovery phase of the PC-like algorithm with two different orderings: $\textrm{order}_1(V)=(d,e,a,c,b)$ and $\textrm{order}_2(V)=(d,c,e,a,b)$. The resulting skeletons are shown in Figures \ref{fig:OrderDepex1LWF}(b) and \ref{fig:OrderDepex1LWF}(c), respectively. 
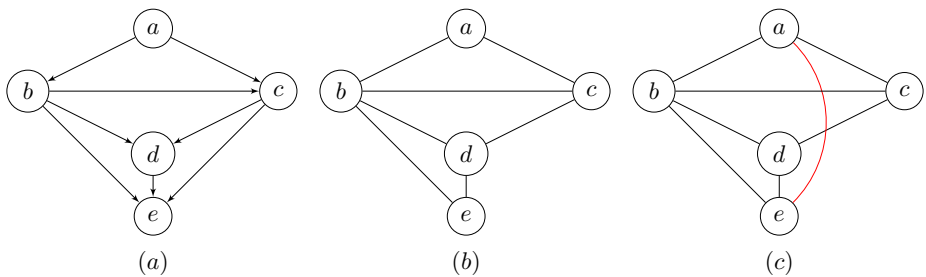
\begin{figure}[!htpb]
    \centering
	\[\resizebox{.8\textwidth}{!}{\begin{tikzpicture}[transform shape]
	\tikzset{vertex/.style = {shape=circle,draw,minimum size=1em}}
	\tikzset{edge/.style = {->,> = latex'}}
	\node[vertex] (o) at  (0,1) {$e$};
	\node[vertex] (p) at  (0,2) {$d$};
	\node[vertex] (q) at  (0,4) {$a$};
	\node[vertex] (r) at  (-2,3) {$b$};
	\node[vertex] (s) at  (2,3) {$c$};
	\node (t) at (0,0.25) {$(a)$};
	\draw[edge] (q) to (r);
	\draw[edge] (q) to (s);
	\draw[edge] (r) to (s);
	\draw[edge] (r) to (p);
	\draw[edge] (r) to (o);
	\draw[edge] (s) to (p);
	\draw[edge] (s) to (o);
	\draw[edge] (p) to (o);
	
	\node[vertex] (i) at  (5,1) {$e$};
	\node[vertex] (j) at  (5,2) {$d$};
	\node[vertex] (k) at  (5,4) {$a$};
	\node[vertex] (l) at  (3,3) {$b$};
	\node[vertex] (m) at  (7,3) {$c$};
	\node (n) at (5,0.25) {$(b)$};
	\draw (j) to (l);
	\draw (j) to (i);
	\draw (i) to (l);
	\draw (k) to (m);
	\draw (k) to (l);
	\draw (j) to (m);
	\draw (l) to (m);
	
	\node[vertex] (e) at  (10,1) {$e$};
	\node[vertex] (d) at  (10,2) {$d$};
	\node[vertex] (a) at  (10,4) {$a$};
	\node[vertex] (b) at  (8,3) {$b$};
	\node[vertex] (c) at  (12,3) {$c$};
	\node (f) at (10,0.25) {$(c)$};
	\draw (c) to (d);
	\draw (e) to (d);
	\draw (e) to (b);
	\draw[red] (a) to [out=315,in=45] (e);
	\draw (b) to (c);
	\draw (a) to (b);
	\draw (a) to (c);
	\draw (b) to (d);
	\end{tikzpicture}}\]
    \caption{(a) The DAG $G$, (b) the skeleton returned by  Algorithm \ref{alg:lwfopc} with $\textrm{order}_1(V)$, (c) the skeleton returned by  Algorithm \ref{alg:lwfopc} with $\textrm{order}_2(V)$. 
}
    \label{fig:OrderDepex1LWF}
\end{figure}

We see that the skeletons are different, and that both are incorrect as the edge $c\erelbar{00}e$ is missing. The skeleton for $\textrm{order}_2(V)$ contains an additional
error, as there is an additional edge $a\erelbar{00}e$. We now go through Algorithm \ref{alg:lwfopc} to see what happened. We start with a complete undirected graph on $V$. When $i= 0$, variables are tested for marginal independence, and the algorithm correctly does not remove any edge. Also, when $i=1$, the algorithm correctly does not remove any edge. When $i= 2$, there is a pair of vertices that is thought to be conditionally independent given a subset of size two, and the
algorithm correctly removes the edge between $a$ and $d$. When $i=3$, there are two pairs of vertices that are thought to be conditionally independent given a subset of size three. Table \ref{t1OrderDepLWFex1}
shows the trace table of Algorithm \ref{alg:lwfopc} for $i=3$ and $\textrm{order}_1(V)=(d,e,a,c,b)$. Table \ref{t2OrderDepMVRex1}
shows the trace table of Algorithm \ref{alg:lwfopc} for $i=3$ and $\textrm{order}_2(V)=(d,c,e,a,b)$.
\begin{table}
\caption{The trace table of Algorithm \ref{alg:lwfopc} for $i=3$ and $\textrm{order}_1(V)=(d,e,a,c,b)$.}\label{t1OrderDepLWFex1}
\centering
\begin{tabular}{c|c|c|c|c}
 Ordered &  & &Is $S_{uv}\subseteq$ & Is $u\erelbar{00} v$ \\
 Pair $(u,v)$ & $ad_H(u)$ & $S_{uv}$ & $ad_H(u)\setminus\{v\}$? & removed?\\
\midrule
\midrule
   $(e,a)$ & $\{a,b,c,d\}$&$\{b,c,d\}$&	Yes&	Yes\\
    \midrule
  $(e,c)$  &$\{b,c,d\}$&$\{a,b,d\}$&	No&	No \\
\midrule
$(c,e)$ &$\{a,b,d,e\}$ &$\{a,b,d\}$&Yes&	Yes\\
\bottomrule
\end{tabular}
\end{table}
\begin{table}
\caption{The trace table of Algorithm \ref{alg:lwfopc} for $i=3$ and $\textrm{order}_2(V)=(d,c,e,a,b)$.}
\centering
\begin{tabular}{c|c|c|c|c}
 Ordered &  & &Is $S_{uv}\subseteq$ & Is $u\erelbar{00} v$ \\
 Pair $(u,v)$ & $ad_H(u)$ & $S_{uv}$ & $ad_H(u)\setminus\{v\}$? & removed?\\
\midrule
\midrule
   $(c,e)$ &$\{a,b,d,e\}$ &$\{a,b,d\}$&Yes&	Yes\\ 
    \midrule
  $(e,a)$ & $\{a,b,d\}$&$\{b,c,d\}$&	No&	No\\
    \midrule
  $(a,e)$  &$\{b,c,e\}$&$\{b,c,d\}$&No&No\\
\bottomrule
\end{tabular}\label{t2OrderDepMVRex1}
\end{table}
\end{example}
\begin{example}\textbf{(Order dependent separating sets and \textit{U}-structures of the PC4LWF algorithm.)}\label{ex2OrderDepLWF}
Suppose that the distribution of $V = \{a,b,c,d,e\}$ is faithful to the DAG $G$ in Figure \ref{fig:OrderDepex2LWF}(a). DAG $G$ encodes the following conditional independencies with minimal separating sets: $a\perp\!\!\!\perp d|b, a\perp\!\!\!\perp e|\{b,c\}, a\perp\!\!\!\perp e|\{c,d\}, b\perp\!\!\!\perp c, b\perp\!\!\!\perp e|d,$ and $c\perp\!\!\!\perp d$.
Suppose that we have an i.i.d. sample of $(a,b,c,d,e)$. Assume that all true conditional independencies are judged to hold except $c\perp\!\!\!\perp d$. Suppose that $c\perp\!\!\!\perp d|b$ and $c\perp\!\!\!\perp d|e$ are thought to hold. Thus, the first is correct, while the second is false. We now apply the complex recovery phase of Algorithm \ref{alg:lwfopc} with two different orderings: $\textrm{order}_1(V)=(d,c,b,a,e)$ and $\textrm{order}_3(V)=(c,d,e,a,b)$. The resulting CGs are shown in Figures \ref{fig:OrderDepex2LWF}(b) and \ref{fig:OrderDepex2LWF}(c), respectively. Note that while the separating set for vertices $c$ and $d$ with $\textrm{order}_1(V)$ is $S_{dc}=S_{cd}=\{b\}$, the separating set for them with $\textrm{order}_2(V)$ is $S_{cd}=S_{dc}=\{e\}$. This illustrates that order dependent separating sets in the skeleton recovery phase of the sample version of Algorithm \ref{alg:lwfopc} can lead to order dependent \textit{U}-structures.
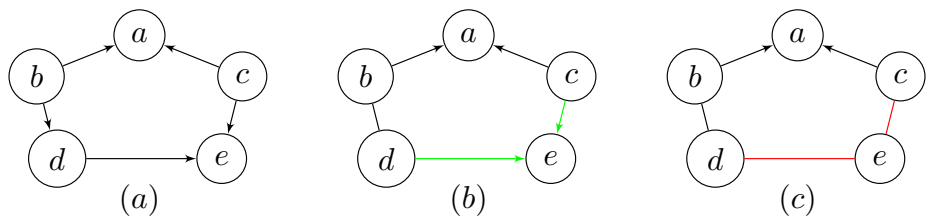
\begin{figure}[ht]
    \centering
	\[\resizebox{.8\textwidth}{!}{\begin{tikzpicture}[transform shape]
	\tikzset{vertex/.style = {shape=circle,draw,minimum size=1em}}
	\tikzset{edge/.style = {->,> = latex'}}
	\node[vertex] (o) at  (1,.5) {$e$};
	\node[vertex] (p) at  (-1,.5) {$d$};
	\node[vertex] (q) at  (0,2) {$a$};
	\node[vertex] (r) at  (-1.25,1.5) {$b$};
	\node[vertex] (s) at  (1.25,1.5) {$c$};
	\node (t) at (0,0) {$(a)$};
	\draw[edge] (r) to (q);
	\draw[edge] (s) to (q);
	\draw[edge] (r) to (p);
	\draw[edge] (s) to (o);
	\draw[edge] (p) to (o);
	
	\node[vertex] (i) at  (5,.5) {$e$};
	\node[vertex] (j) at  (3,.5) {$d$};
	\node[vertex] (k) at  (4,2) {$a$};
	\node[vertex] (l) at  (2.75,1.5) {$b$};
	\node[vertex] (m) at  (5.25,1.5) {$c$};
	\node (n) at (4,0) {$(b)$};
	\draw[edge] (l) to (k);
	\draw[edge,green] (j) to (i);
	\draw[edge] (m) to (k);
	\draw[edge,green] (m) to (i);
	\draw (j) to (l);
	
	\node[vertex] (e) at  (9,.5) {$e$};
	\node[vertex] (d) at  (7,.5) {$d$};
	\node[vertex] (a) at  (8,2) {$a$};
	\node[vertex] (b) at  (6.75,1.5) {$b$};
	\node[vertex] (c) at  (9.25,1.5) {$c$};
	\node (f) at (8,0) {$(c)$};
	\draw[edge] (b) to (a);
	\draw[edge] (c) to (a);
	\draw[red] (e) to (d);
	\draw[red] (e) to (c);
	\draw (b) to (d);
	\end{tikzpicture}}\]
    \caption{(a) The DAG $G$, (b) the CG returned after the complex recovery phase of  Algorithm \ref{alg:lwfopc} with $\textrm{order}_1(V)$, (c) the CG returned after the complex recovery phase of  Algorithm \ref{alg:lwfopc} with $\textrm{order}_3(V)$.}\label{fig:OrderDepex2LWF}
\end{figure}
\end{example}

We now propose several modifications of the original PC-like algorithm for learning LWF chain graphs (and hence also of the related algorithms) that remove the order dependence in the various stages of the algorithm, analogously to what ~\citep{Colombo2014} did for the original PC algorithm in the case of DAGs. 
\subsection{Order Independent Skeleton Recovery}
We first consider estimation of the skeleton in the adjacency search of the PC4LWF algorithm. The pseudocode for our modification is given in Algorithm \ref{alg:lwfspc}. The resulting algorithm is called \textit{\textit{SPC4LWF} (stable PC-like for LWF CGs)}.
The main difference between Algorithms \ref{alg:lwfopc} and \ref{alg:lwfspc} is given by the for-loop on lines
3-5 in the latter one, which computes and stores the adjacency sets $a_H(v_i)$ of all variables after each new size $i$ of the conditioning sets. These stored adjacency sets $a_H(v_i)$ are used
whenever we search for conditioning sets of this given size $i$. Consequently, an edge deletion on line 10 no longer affects which conditional independencies are checked for other pairs of
variables at this level of $i$.
In other words, at each level of $i$, Algorithm \ref{alg:lwfspc} records which edges should be removed, but for the purpose of the adjacency sets it removes these edges only when it goes to the
next value of $i$. Besides resolving the order dependence in the estimation of the skeleton,
our algorithm has the advantage that it is easily parallelizable at each level of $i$  i.e., computations required for $i$-level can be performed in parallel. As a result, the runtime of the parallelized stable PC-like algorithm is much shorter than the original PC-like algorithm for learning LWF chain graphs. Furthermore, this approach enjoys the advantage of knowing the number of CI tests of each level in advance. This allows the CI tests to be evenly distributed over different cores, so that the parallelized algorithm can achieve maximum possible speedup.
The stable PC-like is 
correct, i.e. it returns an LWF CG to which the given probability distribution is faithful (Theorem \ref{thm:correctPCstableLWF}), and it yields order independent skeletons in the sample version (Theorem \ref{thm:stableskeletonLWF}). We illustrate the algorithm in Example \ref{ex:stableskeletonsLWF}.
\begin{algorithm}[t]
\caption{The order independent (stable) PC-like algorithm for learning LWF CGs.}\label{alg:lwfspc}
	\SetAlgoLined
	\small\KwIn{A set $V$ of nodes and a probability distribution $p$ faithful to an unknown LWF CG $G$ and an ordering order($V$) on the variables.}
	\KwOut{The pattern of G}
    Let $H$ denote the complete undirected graph over $V=\{v_1,\dots,v_n\}$\;
\tcc{Skeleton Recovery}
\For{$i\gets 0$ \KwTo $|V_H|-2$}{
\tikzmk{A}
    \For{$j\gets 1$ \KwTo $|V_H|$}{
        Set $a_H(v_j)=ad_H(v_j)$\;
    }\tikzmk{B}\boxit{green!50}
        \While{possible}{
            Select any ordered pair of nodes $u$ and $v$ in $H$ such that $u\in a_H(v)$ and $|a_H(u)\setminus v|\ge i$\, using order($V$);
            
            \If{\textrm{there exists $S\subseteq (a_H(u)\setminus v)$ s.t. $|S|=i$ and $u\perp\!\!\!\perp_p v|S$ (i.e., $u$ is independent of $v$ given $S$ in the probability distribution $p$)}}{
                Set $S_{uv} = S_{vu} = S$\;
                Remove the edge $u \erelbar{00} v$ from $H$\;
            }
        }
    }
    \tcc{Complex Recovery and orientation rules}
    Follow the same procedures in Algorithm \ref{alg:lwfopc} (lines: 11-19). 
\end{algorithm}
\begin{restatable}{theorem}{firstthm}\label{thm:correctPCstableLWF}
    Let the distribution of $V$ be faithful to an LWF CG $G$, and assume that we are given perfect conditional independence information about all pairs of variables $(u,v)$ in $V$ given subsets $S\subseteq V\setminus \{u,v\}$. Then the output of the stable PC-like algorithm is the pattern of $G$.
\end{restatable}

\begin{restatable}{theorem}{secondthm}\label{thm:stableskeletonLWF}
    The skeleton resulting from the sample version of the stable PC-like algorithm is order independent.
\end{restatable}
\begin{example}[Order independent skeletons]\label{ex:stableskeletonsLWF}
We go back to Example \ref{ex1OrderDepLWF}, and consider the sample version of Algorithm \ref{alg:lwfspc}. The algorithm now outputs the skeleton shown in Figure \ref{fig:OrderDepex1LWF}(b) for both orderings $\textrm{order}_1(V)$ and $\textrm{order}_2(V)$. We again go through the algorithm step by step. We start with a complete undirected
graph on $V$. No conditional independence found when $i=0$. Also, when $i=1$, the algorithm correctly does not remove any edge. When $i= 2$, the algorithm first computes the new adjacency sets: $a_H(v)=V\setminus\{v\}, \forall v\in V$. There is a pair of variables that is thought to be conditionally independent given a subset of size two, namely $(a,d)$. Since the sets $a_H(v)$ are not updated after
edge removals, it does not matter in which order we consider the ordered pair. Any ordering leads to the removal of edge between $b$ and $c$.  When $i= 3$, the algorithm first computes the new adjacency sets: $a_H(a)=a_H(d)=\{b,c,e\}$ and $a_H(v)=V\setminus\{v\}, \textrm{ for } v=b,c,e$. There are two pairs of variables that are thought to be conditionally independent given a subset of size three, namely $(a,e)$ and $(c,e)$. Since the sets $a_H(v)$ are not updated after
edge removals, it does not matter in which order we consider the ordered pair. Any ordering leads to the removal of both edges $a\erelbar{00}e$ and $c\erelbar{00}e$. 
\end{example}

\subsection{Order Independent Complex Recovery}
We propose two methods to resolve the order dependence  in the determination of the minimal complexes in LWF CGs, by extending the proposed approaches in \citep{Ramsey:2006} and \citep{Colombo2014} for unshielded colliders recovery in DAGs.

The \textbf{Conservative PC-like algorithm  (CPC4LWF algorithm)} works as follows. Let $H$ be the undirected graph resulting from the skeleton recovery phase  of Algorithm \ref{alg:lwfopc}. For each vertex pair $\{u,v\}$ s.t. $u$ and $v$ are not adjacent in $H$, determine all subsets $S$ of $ad_H(u)$ that make $u$ and $v$ conditionally independent, i.e., that
satisfy $u\perp\!\!\!\perp_p v|S$. We refer to such sets as separating sets. The undirected edge $u\erelbar{00}w$ is labelled as \textit{unambiguous} if at least one such separating set is found and either for each $S$ the set $S\cup\{w\}$ $c$-separates $u$ from $v$ or for none of them $S\cup\{w\}$ $c$-separates $u$ from $v$; otherwise it is labelled as \textit{ambiguous}. If $u\erelbar{00}w$ is
unambiguous, it is oriented as $u\erelbar{01}w$ if and only if for none of the separating
sets $S$, $S\cup\{w\}$ $c$-separates $u$ from $v$. Moreover, in the complex recovery phase of Algorithm \ref{alg:lwfopc}, lines 3-11, the orientation rule is
adapted so that only unambiguous undirected edges are oriented. The output of the CPC4LWF algorithm
is a chain graph in which ambiguous undirected edges are marked. 
We refer to the combination of the SPC4LWF and CPC4LWF algorithms as the \textit{stable CPC4LWF algorithm}.

In the case of DAGs, the authors of \citep{Colombo2014} found that the CPC-algorithm can be very conservative, in the sense that very few
unshielded triples ($v$-structures) are unambiguous in the sample version,  where conditional independence relationships have to be estimated from
data. They proposed a minor
modification of the CPC approach, called \textit{Majority rule PC algorithm (MPC)} to mitigate the (unnecessary) severity of CPC approach.  We similarly propose the \textbf{Majority rule PC-like algorithm (MPC4LWF)} for LWF CGs. As in the CPC4LWF algorithm, we
first determine all subsets $S$ of $ad_H(u)$ that make non adjacent vertices $u$ and $v$ conditionally independent, i.e., that
satisfy $u\perp\!\!\!\perp_p v|S$. The undirected edge $u\erelbar{00}w$ is labelled as \textit{($\alpha, \beta$)-unambiguous} if at least one such separating set is found
or no more than $\alpha$\% or no less than $\beta$\% of sets $S\cup\{w\}$ $c$-separate $u$ from $v$, for $0 \le \alpha\le \beta \le 100$. Otherwise it is labelled as \textit{ambiguous}. (As an example, consider $\alpha = 30$ and $\beta = 60$.) If an undirected edge $u\erelbar{00}w$ is unambiguous, it is oriented as $u\erelbar{01}w$ if and only if less than $\alpha$\% of the sets $S\cup\{w\}$  $c$-separate $u$ from $v$. As in the CPC4LWF algorithm, the orientation rule in the complex recovery phase of the PC4LWF algorithm (Algorithm \ref{alg:lwfopc}, lines 11-18) is adapted so that only unambiguous undirected edge $u\erelbar{00}w$ are oriented, and the output is a chain graph in
which ambiguous undirected edge $u\erelbar{00}w$ are marked.
Note that the CPC4LWF algorithm is the special case of the MPC4LWF algorithm with $\alpha = 0$ and $\beta = 100$.
We refer to the combination of the SPC4LWF and MPC4LWF algorithms as the \textit{stable MPC4LWF algorithm}. Table \ref{Relations:modifiedalgs} summarize our results in section \ref{sec:SPSLWF}.
\begin{table}[t]
\footnotesize\caption{Order dependence issues and corresponding modifications of the PC-like algorithm for LWF CGs
that remove the problem. ``Yes" indicates that the corresponding aspect
of the graph is estimated order independently in the sample version.}\label{Relations:modifiedalgs}
\centering
\resizebox{!}{.065\textheight}{
\begin{tabular}{c|c|c}
 & skeleton recovery & complex recovery \\
\midrule
\midrule
PC-like for LWF CGs & No & No\\
\midrule
stable PC-like for LWF CGs  & Yes & No\\
\midrule
stable CPC/MPC-like for LWF CGs  & Yes & Yes \\
\bottomrule
\end{tabular}}
\end{table}
\begin{restatable}{theorem}{thirdthm}\label{thm:correctCPCstableLWF}
    Let the distribution of $V$ be faithful to an LWF CG $G$, and assume that we are given perfect conditional independence information about all pairs of variables $(u,v)$ in $V$ given subsets $S\subseteq V\setminus \{u,v\}$. Then the output of the (stable) CPC/MPC-like algorithm is the pattern of $G$.
\end{restatable}
\begin{restatable}{theorem}{fourththm}\label{thm:stablevstructsLWF}
    The decisions about \textit{U}-structures in the sample version of the stable CPC/MPC-like algorithm are order independent. In addition, the sample versions of stable CPC-like and stable MPC-like algorithms are fully order independent.
\end{restatable}

\section{Evaluation}\label{evaluation}
To investigate the performance of the proposed algorithms, we use the same approach as in \citep{mxg} for
evaluating the performance of the LCD algorithm on Gaussian LWF CGs. We run our algorithms and the LCD algorithm 
on randomly generated LWF CGs and we compare the results and report summary error measures.
We evaluate the performance of the proposed algorithms in terms of the six measurements that are commonly used~\citep{Colombo2014,Kalisch07,mxg,Tsamardinos2006} for constraint-based learning algorithms, and we report on the first five measurements, due to space limits (see Appendix \ref{appendixC} for a more detailed report): (a) the true positive
rate (TPR) (also known as sensitivity, recall, and hit rate), (b) the false positive rate (FPR) (also known as fall-out), (c) the true discovery rate (TDR) (also known as precision or positive predictive value), (d) accuracy (ACC) for the skeleton, (e) the structural Hamming distance (SHD) (this is the metric described in \citep{Tsamardinos2006} to  compare the
structure of the learned and the original graphs), and (f) run-time for the pattern recovery algorithms. In principle, large values of TPR, TDR, and ACC, and small values of FPR and SHD indicate good performance.

\begin{figure}
\thisfloatpagestyle{empty} 
	\centering
	\includegraphics[scale=.45,page=1]{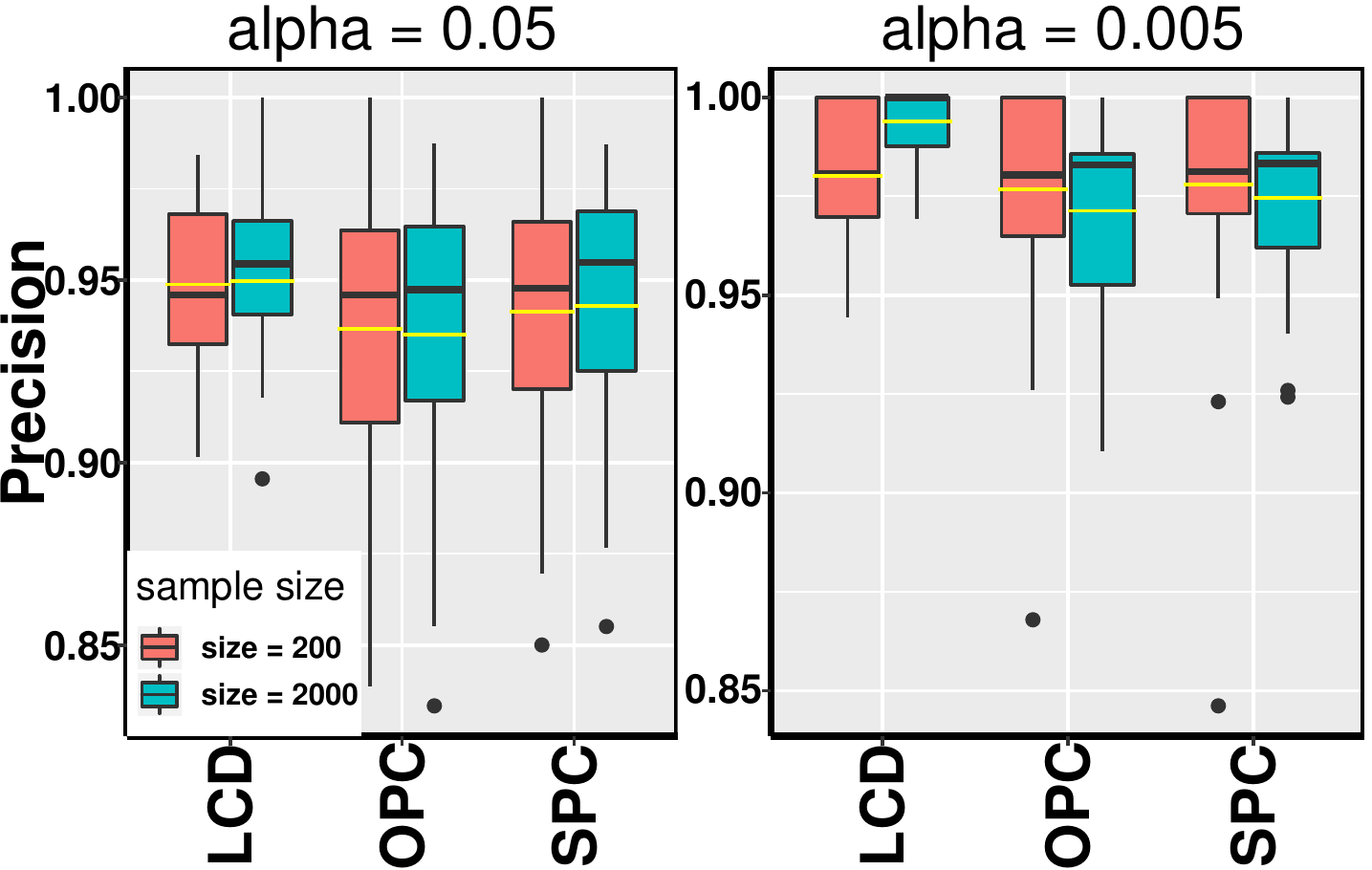}
	\includegraphics[scale=.45,page=1]{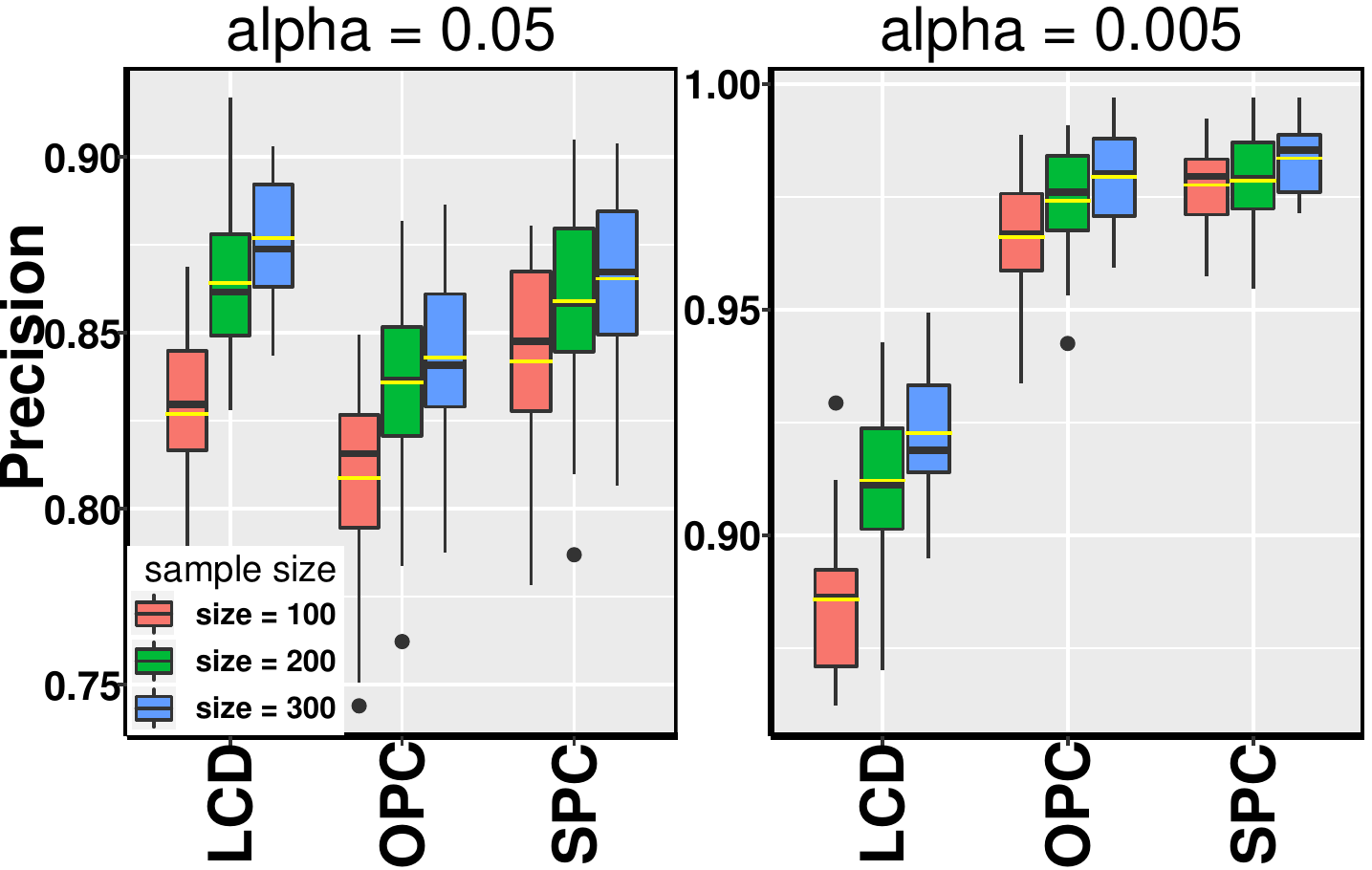}
	\includegraphics[scale=.45,page=2]{images/boxplots503.pdf}
	\includegraphics[scale=.45,page=2]{images/3003boxplots.pdf}
	\includegraphics[scale=.45,page=3]{images/boxplots503.pdf}
	\includegraphics[scale=.45,page=3]{images/3003boxplots.pdf}
	\includegraphics[scale=.45,page=4]{images/boxplots503.pdf}
	\includegraphics[scale=.45,page=4]{images/3003boxplots.pdf}
	\includegraphics[scale=.45,page=5]{images/boxplots503.pdf}
	\includegraphics[scale=.45,page=5]{images/3003boxplots.pdf}
	\caption{Performance of the LCD and PC-like algorithms (original (OPC) and stable (SPC)) for randomly generated Gaussian chain graph models:
	 over 30 repetitions with 50 (the first two columns) and 300 (the last two columns) variables, expected degree N = 3, and significance levels $\alpha=0.05,0.005$.}
	\label{fig:503}
\end{figure}
The first two columns in Figure \ref{fig:503} illustrate the performance of the algorithms in a low dimensional setting with 50 variables and samples of size 200 and 2000.  The last two columns in Figure \ref{fig:503} illustrate the performance of the algorithms in a high-dimensional setting with 300 variables and samples of size 100, 200, and 300.
Figure \ref{fig:503} shows that: \textbf{(a)} in almost all cases, while the performance of the LCD algorithm based on all error measures is better than the performance of the PC-like algorithms in the low-dimensional setting, the performance of the PC-like algorithms (especially the stable PC-like)in the high-dimensional setting are better  than the LCD algorithm (except for the FPR with the p-value $\alpha=0.05$). This indicates that the (stable) PC-like algorithm is computationally feasible and very likely statistically consistent in high-dimensional and sparse setting. \textbf{(b)} The stable PC-like shows similar performance in the low-dimensional setting and improved performance in the high-dimensional setting against the original PC-like algorithm, in particular for error measures precision and FPR (and SHD with the p-value $\alpha=0.05$). \textbf{(c)} In general, the p-value has a very large impact on the performance of the algorithms. Our empirical results suggests that in order to obtain a better precision, FPR, accuracy, and SHD, one can choose a small value (say $\alpha=0.005$) for the significance level of individual tests. \textbf{(d)} While the four error measures TPR, TDR, ACC, and the SHD show a
clear tendency with increasing sample size, the behavior of FPR is not so clear. The latter seems
surprising at first sight but notice that differences are very small with no meaningful indication about the behavior of FPR based on the sample size. 

In summary, empirical simulations show that our proposed algorithms achieve competitive results with the LCD learning algorithm; in particular, in the Gaussian case the SPC4LWF algorithm achieves output of better quality than the LCD and the original PC4LWF algorithm, especially in high-dimensional sparse settings. Besides resolving the order dependence problem,
the SPC4LWF has the advantage that it is easily \textit{parallelizable} and very likely \textit{consistent in high-dimensional settings} (conditions for the consistency would need to be investigated as future work) under the same conditions as the original PC4LWF algorithm.

%
%
%

\appendix
\acks{This work has been supported by AFRL and DARPA (FA8750-16-2-0042). This work is also partially supported by an ASPIRE grant from the Office of the Vice President for Research at the University of South Carolina.}

\section{Correctness of Algorithm \ref{alg:lwfopc}.}\label{appendixA}

Before proving the correctness of the Algorithm \ref{alg:lwfopc}, we need several lemmas.
\begin{lemma}\label{lem1lwfpc}
After line 10 of Algorithm \ref{alg:lwfopc}, $G$ and $H$ have the
same adjacencies.
\end{lemma}
\begin{proof}
Consider any pair of nodes $A$ and $B$
in $G$. If $A\in ad_G(B)$, then $A\not\perp\!\!\!\perp B|S$ for all $S\subseteq V\setminus (A\cup B)$ by the faithfulness assumption. Consequently, $A\in ad_H(B)$ at all times. On the other hand, if $A\not\in ad_G(B)$ (equivalently $B\not\in ad_G(A)$), Algorithm \ref{alg2sepLWF} \citep{jv-pgm18} returns a set $Z\subseteq ad_H(A)\setminus B$ (or $Z\subseteq ad_H(B)\setminus A$) such that $A\perp\!\!\!\perp_p B|Z$.  This means there exist $0\le i\le |V_H|-2$ such that the edge $A-B$ is removed from $H$ in line 7. Consequently, $A\not\in ad_H(B)$ after line 10.
\end{proof}
\begin{algorithm}
\caption{Minimal separation}\label{alg2sepLWF}
	\SetAlgoLined
	\KwIn{Two non-adjacent nodes $A, B$ in the LWF chain graph $G$.}
	\KwOut{Set $Z$, that is a minimal separator for $A, B$.}
	Construct $G_{An(A\cup B)}$\;
	Construct $(G_{An(A\cup B)})^m$\;
	Set $Z'$ to be $ne(A)$ (or $ne(B)$) in $(G_{An(A\cup B)})^m$\;
	\tcc{$Z'$ is a separator because, according to the local Markov property of an undirected graph, a vertex is conditionally independent of all other vertices in the graph, given its neighbors \citep{l}.}
	Starting from $A$, run BFS. Whenever a node in $Z'$ is met, mark it if it is not already marked, and do not continue along
	that path. When BFS stops, let $Z''$ be the set of nodes which are marked. Remove all markings\;
	Starting from $B$, run BFS. Whenever a node in $Z''$ is met, mark it if it is not already marked, and do not continue along
	that path. When BFS stops, let  $Z$ be the set of nodes which are marked\;
	\Return{$Z$}\;
\end{algorithm}

\begin{lemma}\label{lem2lwfpc}
$G$ and $H^*$ have the same minimal complexes and adjacencies after line 19 of Algorithm \ref{alg:lwfopc}.
\end{lemma}
\begin{proof}
$G$ and $H^*$ have the same adjacencies by Lemma \ref{lem1lwfpc}. Now we show that any arrow that belongs to a minimal complex in $G$ is correctly oriented in line 15 of Algorithm \ref{alg:lwfopc}, in the sense that it is an arrow
with the same orientation in $G$. For this purpose, consider the following two cases:

\textbf{Case 1:} $u\to w\gets v$ is an induced subgraph in $G$. So, $u, v$ are not adjacent in $H$ (by Lemma \ref{lem1lwfpc}), $u-w\in H^*$ (by Lemma \ref{lem1lwfpc}), and $u\not\perp\!\!\!\perp_p v|(S_{uv}\cup \{w\})$ by the faithfulness assumption. So, $u - w$ is oriented as $u\to w$ in $H^*$ in line 15. Obviously, we will not orient it as $w\to u$.

\textbf{Case 2:} $u\to w-\cdots-z\gets v$, where $w\ne z$ is a minimal complex in $G$. So, $u, v$ are not adjacent in $H$ (by Lemma \ref{lem1lwfpc}), $u-w\in H^*$ (by Lemma \ref{lem1lwfpc}), and $u\not\perp\!\!\!\perp_p v|(S_{uv}\cup \{w\})$ by the faithfulness assumption. So, $u - w$ is oriented as $u\to w$ in $H^*$ in line 15. Since $u\in S_{vw}$ and $w\perp\!\!\!\perp_p v|(S_{wv}\cup \{u\})$ by the faithfulness assumption  so $u,v$, and $w$ do not satisfy the conditions and hence we will not orient $u - w$ as $w\to u$.

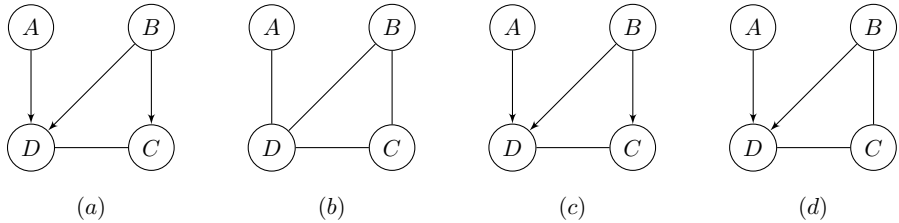
\begin{figure}[ht]
    \centering
	\[\begin{tikzpicture}[scale=.8, transform shape]
	\tikzset{vertex/.style = {shape=circle,draw,minimum size=1.5em}}
	\tikzset{edge/.style = {->,> = latex'}}
	\node[vertex] (o) at  (5,-2) {$B$};
	\node[vertex] (p) at  (3,-4) {$D$};
	\node[vertex] (q) at  (3,-2) {$A$};
	\node[vertex] (r) at  (5,-4) {$C$};
	\node (s) at (4, -5) {$(a)$};
	\draw[edge] (q) to (p);
	\draw[edge] (o) to (r);
	\draw[edge] (o) to (p);
	\draw (p) to (r);
	
	\node[vertex] (u) at  (9,-2) {$B$};
	\node[vertex] (v) at  (7,-4) {$D$};
	\node[vertex] (x) at  (7,-2) {$A$};
	\node[vertex] (w) at  (9,-4) {$C$};
	\node (y) at (8, -5) {$(b)$};
	\draw (x) to (v);
	\draw (u) to (v);
	\draw (u) to (w);
	\draw (v) to (w);
	
	\node[vertex] (a) at  (13,-2) {$B$};
	\node[vertex] (b) at  (11,-4) {$D$};
	\node[vertex] (c) at  (11,-2) {$A$};
	\node[vertex] (d) at  (13,-4) {$C$};
	\node (e) at (12, -5) {$(c)$};
	\draw[edge] (a) to (d);
	\draw[edge] (a) to (b);
	\draw[edge] (c) to (b);
	\draw (b) to (d);
	
	\node[vertex] (f) at  (17,-2) {$B$};
	\node[vertex] (g) at  (15,-4) {$D$};
	\node[vertex] (h) at  (15,-2) {$A$};
	\node[vertex] (i) at  (17,-4) {$C$};
	\node (j) at (16, -5) {$(d)$};
	\draw (f) to (i);
	\draw[edge] (h) to (g);
	\draw[edge] (f) to (g);
	\draw (g) to (i);
	
	\end{tikzpicture}\]
    \caption{(a) The LWF CG $G$, (b) the skeleton of $G$, (c) $H^*$ before executing the line 19 in Algorithm \ref{alg:lwfopc}, and (d) $H^*$ after executing the line 19 in Algorithm \ref{alg:lwfopc}. 
}
    \label{patternlwfpc}
\end{figure}

Consider the chain graph $G$ in Figure \ref{patternlwfpc}(a). After applying the skeleton recovery of Algorithm \ref{alg:lwfopc}, we obtain $H$, the skeleton of $G$, in Figure \ref{patternlwfpc}(b). In the execution of the complex recovery of
Algorithm \ref{alg:lwfopc}, when we pick $A,B$ in line 12 and $C$ in line 13, we have $A\perp\!\!\!\perp B|\emptyset$, that is, $S_{AB}=\emptyset$, and
find that $A\not\perp\!\!\!\perp B|C$. Hence we orient $B-C$ as $B\to C$ in line 15, which is not a complex arrow in $G$.
Note that we do not orient $C-B$ as $C\to B$: the only chance we might do so is when $u = C, v = A$, and $w = B$ in the inner loop of the complex recovery of Algorithm \ref{alg:lwfopc}, but we have $B\in S_{AC}$ and the condition in line 14 is not
satisfied. Hence, the graph we obtain before the last step of complex recovery in Algorithm \ref{alg:lwfopc} must be the one given in
Figure \ref{patternlwfpc}(c), which differs from the recovered pattern in Figure \ref{patternlwfpc}(d). This illustrates the necessity of the last step of complex recovery in Algorithm \ref{alg:lwfopc}. To see how the edge $B\to C$ is removed in the last step of complex recovery in Algorithm \ref{alg:lwfopc}, we observe that, if we follow the procedure described in the comment after line 19 of Algorithm \ref{alg:lwfopc}, the only chance
that $B\to C$ becomes one of the candidate complex arrow pair is when it is considered together with $A\to D$. However, the only undirected path between $C$ and $D$ is simply $C-D$ with $D$ adjacent to $B$.
Hence $B\to C$ stays unlabeled and will finally get removed in the last step of complex recovery in Algorithm \ref{alg:lwfopc}.

Consequently, $G$ and $H^*$ have the same minimal complexes and adjacencies after line 19.
\end{proof}

\section{Proof of Theorems in Section 4.}\label{appendixB}

\firstthm*
\begin{proof}
	The proof of Theorem \ref{thm:correctPCstableLWF} is completely analogous to the proof of the correctness of the original PC-like algorithm (see Appendix \ref{appendixA}).
\end{proof}

\secondthm*
\begin{proof}
	We consider the removal or retention of an arbitrary edge $u\erelbar{00} v$ at some level $i$.
	The ordering of the variables determines the order in which the edges (line 7 of Algorithm \ref{alg:lwfspc}) and the subsets $S$ of $a_H(u)$ and $a_H(v)$ (line 8 of Algorithm \ref{alg:lwfspc}) are considered. By
	construction, however, the order in which edges are considered does not affect the sets $a_H(u)$ and $a_H(v)$.
	
	If there is at least one subset $S$ of $a_H(u)$ or $a_H(v)$ such that $u\perp\!\!\!\perp_p v|S$, then any
	ordering of the variables will find a separating set for $u$ and $v$. (Different orderings
	may lead to different separating sets as illustrated in Example 2, but all edges that have a separating set will eventually be removed, regardless of the ordering). Conversely, if there is no subset $S'$ of $a_H(u)$ or $a_H(v)$ such that $u\perp\!\!\!\perp_p v|S'$, then no ordering will find a separating set.
	
	Hence, any ordering of the variables leads to the same edge deletions, and therefore to
	the same skeleton.
\end{proof}

\thirdthm*
\begin{proof}
	The skeleton of the learned pattern is correct by Theorem \ref{thm:correctPCstableLWF}.
	Since $u, v$ are not adjacent they are \textit{c}-separated given some subset $S\setminus\{u,v\}$ (see Algorithm \ref{alg2sepLWF}). Based on the \textit{c}-separation criterion for LWF CGs (see section 2), if $w$ is a node on a minimal complex in $G$ such that $u$ and $w$ are adjacent then $u\not v |S\cup\{w\}$ for any $S\setminus\{u,v\}$ due to the moralization procedure.  
	As a result, $u-w$ edges are all unambiguous and so the $U$-structures are correct as in the CPC/MPC-like algorithm. Therefore,  the output of the (stable) CPC/MPC-like algorithm is a pattern that is Markov equivalent with $G$.
\end{proof}

\fourththm*
\begin{proof}
	The stable CPC/MPC-like algorithm have order independent skeleton, by
	Theorem \ref{thm:stableskeletonLWF}. In particular, this means that their adjacency sets are order independent. For non adjacent nodes $u$ and $v$ the decision about whether the undirected edge $u-w$ is unambiguous
	and/or a \textit{U}-structure is based on the adjacency sets of nodes $u$ and $v$, which are order independent. The rest of theorem follows straightforwardly from Theorems \ref{thm:stableskeletonLWF} and the first part of this proof.
\end{proof}

\section{Evaluation}\label{appendixC}
In this section, we evaluate the performance of our algorithm in various setups
using simulated / synthetic data sets. We first compare the performance of our algorithms with the LCD algorithm \citep{mxg} by running them
on randomly generated LWF CGs. 
Empirical simulations show that our PC-like algorithms achieves competitive results with the LCD algorithm in terms of error measures and runtime. 
All the results reported here are
based on our R implementations.
The R code and results are reported in our supplementary materials available at: \url{https://github.com/majavid/PC4LWF2020}.

\subsection{Performance on Random LWF CGs}
To investigate the performance of the proposed algorithms, we use the same approach as in \citep{mxg} for
evaluating the performance of the LCD algorithm on LWF CGs. We run our algorithms PC4LWF \& SPC4LWF, and the LCD algorithm
on randomly generated LWF CGs and we compare the results and report summary error measures.

\subsubsection{Data Generation Procedure}
First we explain the way in which the random LWF CGs and random samples are generated.
Given a vertex set $V$, let $p = |V|$ and $N$ denote the average degree of edges (including undirected, pointing out, and pointing in) for each vertex. We generate a random LWF CG on $V$ as
follows:
\begin{enumerate}
    \item Order the $p$ vertices and initialize a $p\times p$ adjacency matrix $A$ with zeros;
    \item  For each element in the lower triangle part of $A$, set it to be a random number generated from
a Bernoulli distribution with probability of occurrence $s = N/(p-1)$;
\item Symmetrize $A$ according to its lower triangle;
\item Select an integer $k$ randomly from $\{1,\dots,p\}$ as the number of chain components;
\item  Split the interval $[1, p]$ into $k$ equal-length subintervals $I_1,\dots,I_k$ so that the set of variables
falling into each subinterval $I_m$ forms a chain component $C_m$; and, 
\item Set $A_{ij} = 0$ for any $(i, j)$ pair such that $i \in I_l, j \in I_m$ with $l > m$.
\end{enumerate}
This procedure yields an adjacency matrix $A$ for a chain graph with $(A_{ij} = A_{ji} = 1)$ representing an undirected edge between $V_i$ and $V_j$ and $(A_{ij} =1,  A_{ji} =0)$ representing a directed edge
from $V_i$ to $V_j$. Moreover, it is not difficult to see that $\mathbb{E}[\textrm{vertex degree}] = N$, where an adjacent vertex can
be linked by either an undirected or a directed edge.

Given a randomly generated chain graph $G$ with ordered chain components $C_1,\dots,C_k$, we generate a Gaussian distribution on it via the $\mathsf{rnorm.cg}$ function from the \href{http://www2.uaem.mx/r-mirror/web/packages/lcd/lcd.pdf}{LCD} R package. 

\subsubsection{Experimental Results}
We evaluate the performance of the proposed algorithms in terms of the six measurements that are commonly used~\citep{Colombo2014,Kalisch07,mxg,Tsamardinos2006} for constraint-based learning algorithms: (a) the true positive
rate (TPR) (also known as sensitivity, recall, and hit rate), (b) the false positive rate (FPR) (also known as fall-out), (c) the true discovery rate (TDR) (also known as precision or positive predictive value), (d) accuracy (ACC) for the skeleton, (e) the structural Hamming distance (SHD) (this is the metric described in \citep{Tsamardinos2006} to  compare the
structure of the learned and the original graphs), and (f) run-time for the pattern recovery algorithms. In short, $TPR$ is the ratio of  the number of correctly identified edges over total number of edges, $FPR$ is the ratio of the number of incorrectly identified edges over total number of gaps, $TDR$ is the ratio of  the number of correctly identified edges over total number of edges (both in estimated graph), $ACC=\frac{\textrm{true positive }(TP) +\textrm{ true negative }(TN)}{Pos+Neg}$, and
$SHD$ is the number of legitimate operations needed to change the current resulting graph to the true CG,
where legitimate operations are: (a) add or delete an edge and (b) insert, delete or reverse an edge
orientation. In principle, a large TPR, TDR, and ACC, a small FPR and SHD indicate good performance.

In our simulation, we change three parameters $p$ (the number of vertices), $n$ (sample size) and
$N$ (expected number of adjacent vertices) as follows:
\begin{itemize}
    \item $p=50$ in low-dimensional settings and $p=300$ in high-dimensional settings,
    \item $n\in\{200, 2000\}$ in low-dimensional settings and $n\in\{100,200,300\}$ in high-dimensional settings, and
    \item $N\in\{2,3\}$.
\end{itemize}
\begin{figure}
	\centering
\thisfloatpagestyle{empty} 
	\centering
	\includegraphics[scale=.45,page=1]{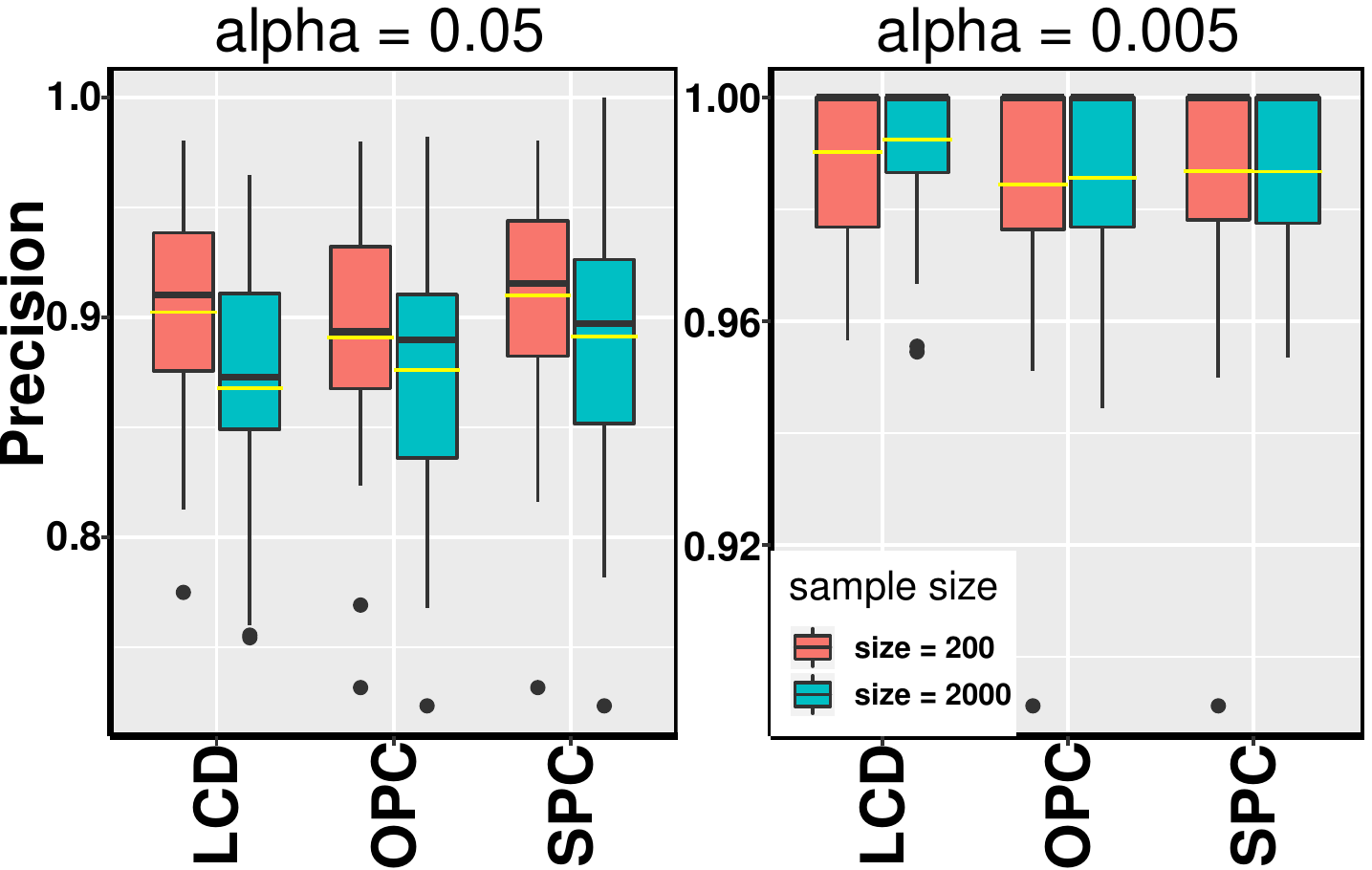}
	\includegraphics[scale=.45,page=1]{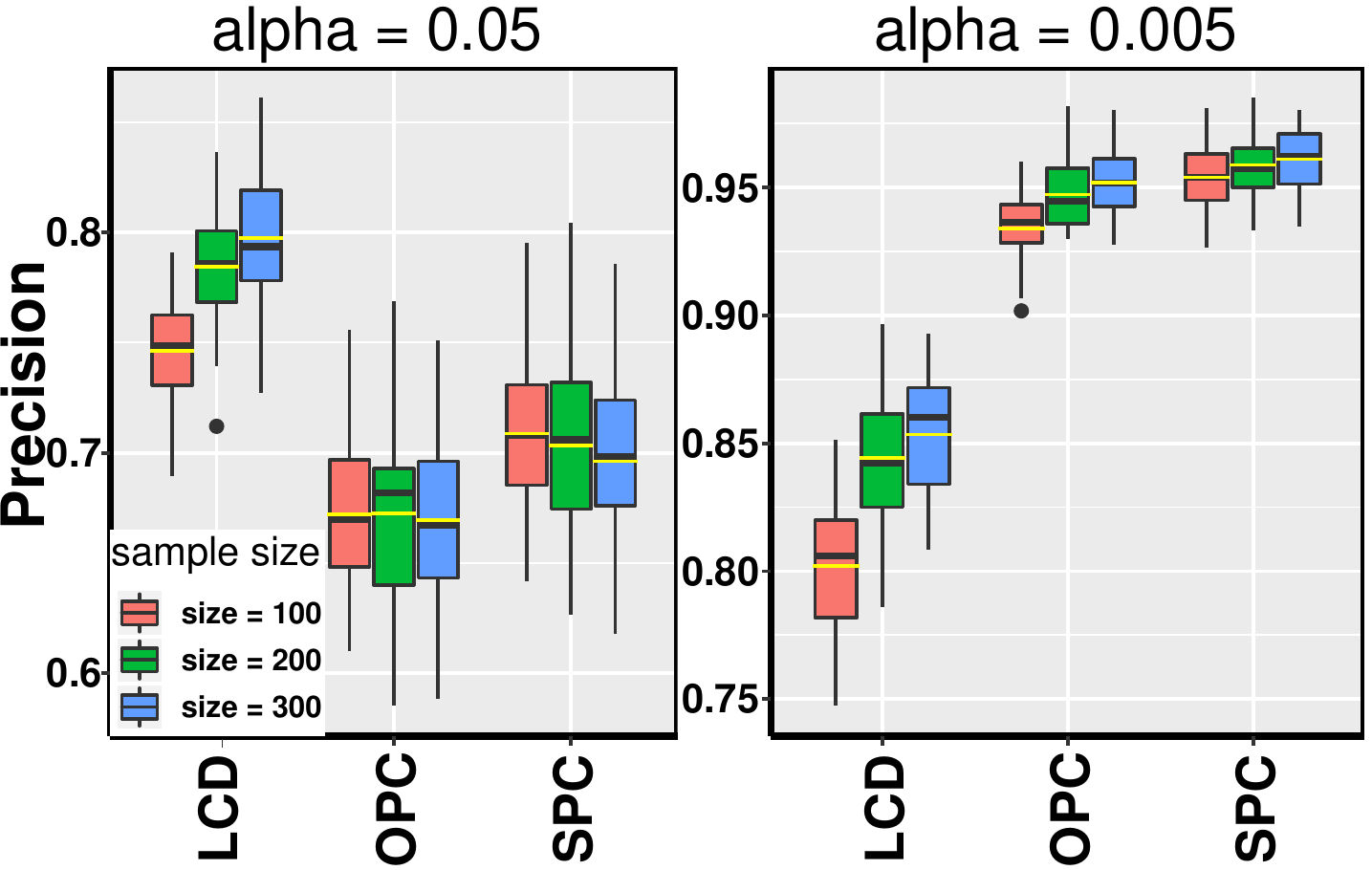}
	\includegraphics[scale=.45,page=2]{images/boxplots502.pdf}
	\includegraphics[scale=.45,page=2]{images/3002boxplots.pdf}
	\includegraphics[scale=.45,page=3]{images/boxplots502.pdf}
	\includegraphics[scale=.45,page=3]{images/3002boxplots.pdf}
	\includegraphics[scale=.45,page=4]{images/boxplots502.pdf}
	\includegraphics[scale=.45,page=4]{images/3002boxplots.pdf}
	\includegraphics[scale=.45,page=5]{images/boxplots502.pdf}
	\includegraphics[scale=.45,page=5]{images/3003boxplots.pdf}
	\caption{Performance of the LCD and PC-like algorithms (original (OPC) and stable (SPC)) for randomly generated Gaussian chain graph models:
	 over 30 repetitions with 50 (the first two columns) and 300 (the last two columns) variables, expected degree N = 2, and significance levels $\alpha=0.05,0.005$.}
	\label{fig:1003a}
\end{figure}
For each $(p,N)$ combination, we first generate 30 random LWF CGs. We then generate a
random Gaussian distribution based on each graph and draw an identically independently distributed
(i.i.d.) sample of size $n$ from this distribution for each possible $n$. For each sample, two different
significance levels $(\alpha = 0.05, 0.005)$ are used to perform the hypothesis tests. The \textit{null hypothesis} $H_0$ is ``two variables $u$ and $v$ are conditionally independent given a set $C$ of variables" and alternative $H_1$ is that $H_0$ may not hold. We then
compare the results to access the influence of the significance testing level on the performance of our algorithms. In order to learn an \textit{undirected independence graph} (UIG) from a given data set in the LCD algorithm we used the stepwise forward selection (FWD-BIC) algorithm \citep{Edwards2010} in high-dimensional settings and the \textit{Incremental Association Markov blanket} discovery (IAMB) algorithm \citep{Tsamardinos0Mb} in low-dimensional settings.

\begin{remark}
	Since both the PC-like algorithm and the LCD algorithm assume faithfulness and the CKES algorithm \citep{psn} does not assume the faithfulness requirement, the comparison between our proposed algorithms and the CKES algorithm may seem unfair (for a detailed discusion see \citep{psn}). Also, we did not compare the proposed algorithms in this paper with the ASP algorithm  due to the scalability issues discussed in \citep{sjph}.  
\end{remark}

\bibliography{references}

\end{document}